\documentclass[letterpaper]{article} 
\usepackage{aaai25}  
\usepackage{times}  
\usepackage{helvet}  
\usepackage{courier}  
\usepackage[hyphens]{url}  
\usepackage{graphicx} 
\usepackage{natbib}  
\usepackage{caption} 
\usepackage{amsmath}
\usepackage{breqn}
\usepackage{amssymb}

\usepackage{amsthm}
\usepackage{booktabs}
\usepackage{enumitem}

\usepackage[ruled,vlined,linesnumbered]{algorithm2e}
\usepackage{algorithmic}
\usepackage[table]{xcolor}
\usepackage{subfig}
\usepackage{mathrsfs}
\usepackage{todonotes}
\newtheorem{definition}{Definition}
\newtheorem{example}{Example}
\newtheorem{proposition}{Proposition}
\newtheorem{lemma}{Lemma}
\title{Tab-Shapley: Identifying Top-$k$ Tabular \\Data Quality Insights}
\author {
    Manisha Padala\textsuperscript{\rm 1},
     Lokesh Nagalapatti\textsuperscript{\rm 2},
     Atharv Tyagi\textsuperscript{\rm 3},
     Ramasuri Narayanam\textsuperscript{\rm 3},
     Shiv Kumar Saini\textsuperscript{\rm 3}
}
\affiliations {
    \textsuperscript{\rm 1}IIT Gandhinagar,
    \textsuperscript{\rm 2} IIT Bombay,
    \textsuperscript{\rm 3}Adobe Research\\
    manisha.padala@iitgn.ac.in,
    nlokeshiisc@gmail.com, athtyagi@adobe.com, rnarayanam@adobe.com, shsaini@adobe.com, 
}
\usepackage{bibentry}

\begin{document}

\maketitle

\begin{abstract}
We present an unsupervised method for aggregating anomalies in tabular datasets by identifying the top-$k$ {\em tabular data quality insights}. Each insight consists of a set of anomalous attributes and the corresponding subsets of records that serve as evidence to the user. The process of identifying these insight blocks is challenging due to (i) the absence of labeled anomalies, (ii) the exponential size of the subset search space, and (iii) the complex dependencies among attributes, which obscure the true sources of anomalies. Simple frequency-based methods fail to capture these dependencies, leading to inaccurate results. To address this, we introduce {\em Tab-Shapley}, a  cooperative game theory based framework that uses Shapley values to quantify the contribution of each attribute to the data's anomalous nature. While calculating Shapley values typically requires exponential time, we show that our game admits a closed-form solution, making the computation efficient. We validate the effectiveness of our approach through empirical analysis on real-world tabular datasets with ground-truth anomaly labels.
\end{abstract}

\section{Introduction}


Anomalies in data can significantly hinder the performance of machine learning models \cite{deep-learning-2021, datashapley}. To address this, various techniques have been developed to detect and remove anomalies \cite{chandola2009, datashapley, auto-detect-2018}. Given the absence of labeled anomalies in training datasets, research has increasingly turned to semi-supervised and unsupervised methods for anomaly detection \cite{pangsur2021}. However, many of these approaches focus on automatic anomaly removal without providing human-readable insights into why the flagged data is considered problematic. Gaining such insights is essential for identifying the sources of anomalies, which in turn allows for effective countermeasures to be deployed. Yet, pinpointing the origin of an anomaly often requires deep domain expertise. Without this knowledge, data engineers may find themselves overwhelmed by the complex patterns they must analyze. The challenge, then, is not just in detecting anomalies but in empowering those who manage the data with the understanding they need to take meaningful action.



Motivated by this, our paper focuses on not only detecting anomalies but also offering valuable {\em data quality insights}. We center our efforts on tabular data, which is widely used across numerous enterprise analytics platforms, making our approach highly relevant and impactful for real-world applications.



\noindent\textbf{Data Quality Insights.} Data Insights are prioritized blocks of data, consisting of specific attributes and records, that contain a high concentration of anomalies.  Inspection of such blocks facilitate effective recovery of anomaly sources.  The higher a block appears on the priority list, the more likely it is to be the source of the anomaly, potentially indicating how the attributes and records are compromised.

We illustrate the notion of a data quality insight using an example. Let us consider the example in Table \ref{tab::eg} (top), where the anomalous cells are highlighted in purple. In the second row, we notice an inconsistency: a person with only primary education at the age of 10 is earning $80K$—clearly suspicious and likely due to a data entry error. Across the entire dataset, the attributes {{\em Income}} and {{\em Occupation}} emerge as the most anomalous. To better visualize these anomalies, we rearrange the table, bringing the anomalies to the top-left corner, as shown in Table \ref{tab::eg} (bottom). These concentrated clusters of anomalies are what we refer to as "data insights." In this example, the highlighted rows {{\em Row 4}} and {{\em Row 5}} within the attributes underscore the anomalous patterns. Given the many possible combinations of anomalous attributes, manually investigating all these block structures can be overwhelming. To tackle this, we propose a solution that generates a \textbf{prioritized} list of data insights by identifying the "top-k block structures" in order of decreasing significance.



\begin{table}[!t]
    \centering
    \begin{tabular}{ccccc}
    \toprule
       Age & Education & \cellcolor{blue!25}Income   & \cellcolor{blue!25}Occupation  \\
    \midrule
        30 & high school & 50K   &  military \\
        10 & primary     & \cellcolor{blue!25}80K    & unemployed\\
        45 & graduate    & 120K   &  manager \\
        70 & graduate    & \cellcolor{blue!25}85K    & retired \\
        20 & high school & \cellcolor{blue!25}280K   & \cellcolor{blue!25}advocate \\
        34 & graduate    & 100K   & \cellcolor{blue!25}unemployed \\
        39 & high school   & 100    &  unemployed \\
    \bottomrule
    \end{tabular}
    \quad \quad
    \begin{tabular}{cccc}
    \toprule
    Income  & Occupation &  Age & Education  \\
    \midrule
    \cellcolor{green}280K  & \cellcolor{green}advocate & 20 & high school\\
        \cellcolor{green}85K   & retired & 70 & graduate   \\
         80K  & unemployed & 10 & primary       \\
         100K   & unemployed & 34 & graduate    \\
         50K    &  military & 30 & high school  \\
         120K  &  manager & 45 & graduate  \\
         100     &  unemployed & 39 & high school \\
    \bottomrule
    \end{tabular}
    \caption{Top table shows dataset with potential anomalous cells colored in purple and not potential anomalous cells uncolored. Bottom table Illustrates  a {\em top-$1$ data insight} (in green color), $<\{$Income, Occupation$\}$, $\{$Row 4, Row 5$\}>$.}
    \label{tab::eg}
\end{table}




More formally, given a tabular data and an integer $k$, we aim to determine {\bf top-$k$ tabular data quality insights (or block structures)} to offer to users, where each {\em data quality insight} comprises of two components: $<$subset of anomalous attributes, subset of records where anomalous behavior can be observed$>$. 

\smallskip
{\noindent\bf Challenges.} The absence of supervision in identifying data quality insights makes this task particularly challenging. Additionally, pinpointing the attributes or records as primary sources of anomalies involves a combinatorial search, which is time-intensive. To overcome these challenges, we model the problem using cooperative game theory and leverage Shapley values to aggregate insights. While the computation of Shapley values typically requires exponential time, we demonstrate that our problem formulation enables an elegant closed-form solution, significantly improving efficiency.


\smallskip

{\noindent\bf Contributions.} Our contributions are as follows: (i) We introduce the novel problem of top-$k$ data quality insights for analyzing anomalies in tabular data, providing a solution that prioritizes insights into potential sources of anomalies. (ii) We propose a cooperative game-theoretic model in Section \ref{labelled_cg_Shapley_value}, defining evidence sets for each attribute and record, and calculating their anomalous scores using Shapley values. (iii) In Section \ref{labelled_cg_Shapley_value}, we present a key analytical result that allows efficient computation of Shapley values through a closed-form expression. (iv) Using these anomalous scores, we reorganize the data into block-like structures and introduce Algorithm \ref{algorithm:top-k-DI-extraction} to efficiently identify the top-$k$ data quality insights. (v) Finally, we demonstrate the effectiveness of our proposed approach through extensive experiments on several real-world datasets.

\section{Related Work}
Most prior art explains feature importance scores or provides reasons for anomalous predictions in images and videos \cite{liznerski2020}. For tabular data, the existing literature on deriving feature importance is limited. Some notable methods (\cite{amarasinghe2018, xu2021,macrobase}) that provide feature importance scores as explanations for anomalies require explicit supervision on anomaly labels which are difficult to acquire in practice. The approaches that come most close to our work are the following. In \cite{pang2021}, the authors consider few-shot learning and learn an end-to-end scoring rule. In \cite{antward21}, the authors identify attributes with high reconstruction errors and provide SHAP-based explanations for each of these attributes, but the focus is on identifying attributes rather than records. In \cite{carletti2019}, the authors provide attribute importance specifically designed for isolation forest-based anomaly detection, while in \cite{nguyen2019}, gradient-based methods for reconstruction loss in a VAE are used to derive attribute importance scores for detecting network intrusions.

The use of game theory in the areas of data engineering is a known art in the literature. Below are a few relevant prior art at this intersection:
(i) \cite{eigen-game-iclr-2021} uses game theory as an engine for large scale data analysis; (ii)  \cite{vldb-j-2011} applies game theory to secure data integration; (iii) \cite{sci-adv-2021} for secure sharing of data; (iv) \cite{acm-db-sys-2020} for modelling dynamic interaction between users and DBMS; and finally (v) \cite{amirata:2019} for data valuation, etc.

\section{Preliminaries}



We first begin with a brief introduction to the relevant concepts from cooperative game theory that we will use in the subsequent sections.  

\smallskip
{\noindent\bf Cooperative Games.} \cite{myerson:1997}: We now formally define the notions of a cooperative game and the Shapley value. Let $N = \{1,2,\ldots,n\}$ be the set of players of a cooperative game. A \textit{characteristic function} $v: 2^N \to\mathbb{R}$ assigns a real number to every coalition $C \subseteq N$ that represents payoff attainable by this coalition. By convention, it is assumed that $v(\emptyset)=0$. Now, the two tuple $(N,v)$ defines the {\em cooperative game} or \textit{characteristic function game}.
We call the characteristic function $v$  {\em super-additive}, if $\forall S,R \subseteq N$ and $S \cap R = \Phi$, i.e., $ v(S \cup R) \geq v(S) + v(R)$.

The consequence of super-additive property is that it ensures the formation of grand coalition, i.e. $v(N)\geq \sum_{i \in N} v({i})$.  

\smallskip
{\noindent\bf Shapley Value.} If the cooperative game is super-additive, the grand coalition (that consists of all the players in the game) forms. Given this, one of the rudimentary questions that cooperative game theory answers is how to distribute the payoff of the grand coalition among the individual players. Towards this end, Shapley \cite{Shapley1971} proposed to evaluate the role of each player in the game by considering its marginal contributions to all coalitions this player could possibly belong to. A certain weighted sum of such marginal contributions constitutes a player's payoff from the coalition game and is called the Shapley value \cite{myerson:1997,straffin:1993}. 
Importantly, Shapley proved that his payoff division scheme is the only one that meets, at the same time, the following four desirable criteria:
\begin{itemize}
    \item[(i)] \emph{efficiency} --- all the payoff of the grand coalition is distributed among players;
    \item[(ii)] \emph{symmetry} --- if two agents play the same role in any coalition they belong to (i.e. they are symmetric) then their payoff should also be symmetric;
    \item[(iii)] \emph{null player} --- agents with no marginal contributions to any coalitions whatsoever should receive no payoff from the grand coalition; and
    \item[(iv)] \emph{additivity} --- values of two uncorrelated games sum up to the value computed for the sum of both games.
\end{itemize}

Formally, let $\pi \in \Pi(N)$ denote a permutation of players in $N$, and let $C_{\pi}(i)$ denote the coalition made of all predecessors of agent $i$ in $\pi$ (if we denote by $\pi(j)$ the location of $j$ in $\pi$, then: $C_{\pi}(i) = \{j \in \pi: \pi(j) < \pi(i)\}$). Then the Shapley value is defined as follows \cite{monderer:1996}:

\begin{equation}\label{originalShapley}
SV_i(v) = \frac {1}{|N|!} \sum_{\pi \in \Pi}[v(C_{\pi}(i) \cup \{i\}) - v(C_{\pi}(i))],
\end{equation}

i.e., the payoff assigned to $a_i$ in a coalitional game is the average marginal contribution of $a_i$ to coalition $C_{\pi}(i)$ over all $\pi \in \Pi$.  It is easy to show that the above  formula can be rewritten as:
\begin{equation}\label{Original_SV2nd_form}
SV_i(v) = \sum_{C \subseteq A \setminus \{i\}}\frac {|C|!(|N|-|C|-1)!} {|N|!} [v(C \cup \{i\}) - v(C)].
\end{equation}

We provide an illustration of the same in the extended version. Given the definitions, we discuss our proposed solution in the next section.

\section{Proposed Solution Approach}

We first set the notation used in our paper. Let $T$ be the tabular data consisting of a set of $n$ records and a set of $m$ attributes denoted as $A=\{a_1,a_2,\ldots, a_m\}$. Each record in $T$ is represented as $X_i=(X_{i1},X_{i2},\ldots, X_{im})$, where $X_{ij}$ is the value of attribute $a_j$ for record $X_i$. Throughout the paper, we use $i$ to index records and $j$ to index attributes. We assume that there exists an error value $e_{ij}$ for each attribute value prediction using unsupervised learning methods, such as auto-encoders. We also define a label $L_{ij}$ for each $X_{ij}$ based on the error value $e_{ij}$. The label is either {\em $N\!A$} (\textbf{N}ot a potential \textbf{A}nomaly) or {\em $P\!A$} (\textbf{P}otential \textbf{A}nomaly).

\subsection{Deriving Labels for Cells of Tabular Data}
\label{subsec:PANA}
The first step of our approach involves labeling individual cells in the data as anomalous or not. To achieve this, we train an auto-encoder on the tabular data and use it to reconstruct missing values. We describe the details below,

\smallskip
\noindent\textbf{Training Auto-Encoder:} 
The auto-encoder is trained using the TABNET \cite{tabnet} framework, which is based on an encoder-decoder architecture. During training, 50\% of the features are randomly masked, and the TABNET predicts only the masked features. 


\smallskip
\noindent\textbf{Cell-level Reconstruction Loss:}
During testing, for each test sample $i$, we mask each attribute $j$ iteratively and use the pre-trained TABNET to predict the masked attribute. The error $e_{ij}$ is calculated as the mean-squared error for continuous attributes and cross-entropy loss for categorical attributes between the predicted value and the actual value of the attribute. To make the loss values comparable across different attributes, we standardize the continuous features, and normalize the categorical features between 0 and 1. 

\smallskip
\noindent\textbf{Thresholding of Records:} 
The record level loss $e_i$ is calculated as the average of cell-level losses $e_{ij}$ for each record $i$. To determine the threshold on these errors for identifying anomalous records, we use clustering. We then assign record-level labels $\hat{y}_i \in {1, 0}$ based on whether $e_i$ is above or below the threshold, respectively. Given record-level predictions, we proceed with attribute-level labels.


\smallskip
\noindent\textbf{Thresholding of Attributes:} 
Here we consider every anomalous record $i$, i.e., $\hat{y}_i = 1$. For each such $i$, we cluster $[e_{ij}]_{j=\{1,\ldots,m\}}$ into two clusters using k-means algorithm. The attributes $j$ belonging to the cluster having higher $e_{ij}$ is labelled anomalous, $L_{ij} = P\!A$; otherwise $L_{ij}=N\!A$. In summary, we obtain the cell level predictions $L_{ij}$, where it takes value $P\!A$ if for an $i$ that is predicted to be anomalous we obtain that the $j$ is anomalous as described above.


It is important to note that the method described above is not the only way to calculate labels $L_{ij}$ for tabular data, approaches besides TABNET can also be used. In the next subsections, we discuss how to compute the aggregated scores at both attribute and record level.


\subsection{Evidence Sets and  Cooperative Game}
\label{evidence-sets-cgs}

Using the cell-level label information, we can define evidence sets for attributes (and records), as discussed below

\begin{definition} 
[Evidence Sets for Attributes:] We define evidence set $E_{a_j}$ for attribute $a_j$ to be the set of all records for which the label of $a_j$ is {\em $N\!A$}. That is,
\begin{equation}
\label{evidence-set-attributes}
    E_{a_j} = \{X_i | L_{ij} = N\!A \ \  \mbox{for any} \ i \in \{1,2,\ldots,n\} \} 
\end{equation}

\end{definition}

\begin{definition} 
[Evidence Sets for Records:] We define evidence set $E_{X_i}$ for record $X_i$ to be the set of all attributes with  {\em $N\!A$} being their respective label. That is, 
\begin{equation}
    E_{X_i} = \{a_j | L_{ij} = N\!A\ \  \mbox{ for any} \ j \in \{1,2,\ldots,m\} \}
\end{equation}

\end{definition}



Our proposed approach for deriving anomalous scores for all the attributes in tabular data $T$ is based on the collection ${E_{a_1}, E_{a_2}, . . . , E_{a_m}}$ of evidence sets corresponding to $m$ attributes. We first compute the \textbf{non-anomalous} score\footnote{In consensus with literature, we compute the non-anomalous scores using Shapley values as pay-offs to the players and then invert them to obtain anomalous scores} for each attribute based on the following criteria:
\begin{itemize}
    \item {\em Criteria 1:} The score of the corresponding attribute should be higher if the size of its evidence set is larger. The size of the evidence set reflects the statistical significance of the respective attribute {\em not being} an anomaly.
    \item  {\em Criteria 2:} The score of the corresponding attribute should be higher if the number of unique records that are part of its evidence set is larger.
\end{itemize}

We now define a cooperative game in order to compute these non-anomalous scores of attributes while capturing the above two criteria. 

\noindent{\bf Cooperative Game based on Attributes:}
Let us define a cooperative game $(A,\mathscr{V}_a)$ based on the attributes of tabular data as follows: (i) The set of players $A$ comprises the attributes; and (ii) $\mathscr{V}_{a}(\bullet): 2^{m} \to \mathbb{R}$ is a characteristic function that assigns a value to each subset of players. For each subset $S \subseteq A$, $\mathscr{V}_{a}(S)$ is defined as the cardinality of the set of all records that are members of at least one evidence set corresponding to the attributes in $S$ given by,
\begin{equation}
    \mathscr{V}_{a}(S) = |\cup_{a_j \in S} E_{a_j}|.
    \label{cg-def-attributes}
\end{equation}


\subsection{Computing Shapley Values}
\label{subsec:csv}
The computation of Shapley values for any given cooperative game is known to be a computationally challenging task as it involves dealing with an exponential number of player subsets. However, the cooperative game $(A,\mathscr{V}_a)$ proposed in our approach has a specific structure that allows us to calculate Shapley values for the players (i.e., attributes) efficiently in polynomial time. This is due to the fact that the game $(A,\mathscr{V}_a)$ satisfies the super-additive property, which can be easily verified. As a result of this property, we can derive a closed-form expression for computing Shapley values using Eqns \ref{originalShapley} and \ref{Original_SV2nd_form}, as we will explain in the following section.

\begin{proposition} 
{\em 
The above defined cooperative game $(A,\mathscr{V}_a)$ is super-additive.
}
\end{proposition}

The following lemma formally proves that the closed form Shapley values of attributes can be computed efficiently.

\begin{lemma} {\em 
In the cooperative game $(A,\mathscr{V}_a)$, the Shapley value $\phi_{a}(a_j)$ of each attribute $a_j \in A$ can be computed as follows:
\begin{equation*}
    \phi_{a}(a_j) = \sum_{X_i \in E_{a_j}} \frac{1}{|\{ k: X_i \in E_{a_k} \}|}
\end{equation*}
\label{labelled_cg_Shapley_value}}
\end{lemma}
\begin{proof}
Recall that, the Shapley value of each attribute $a_j$ using the permutation based definition is as follows:
\begin{equation*}
    \phi(a_j) = \frac{1}{m!} \sum_{R} \big[\mathscr{V}_a(P_R^{a_j} \cup \{a_j\}) - \mathscr{V}_a(P_R^{a_j})\big]
    \label{lemma2-proof-eqn1}
\end{equation*}
where the sum ranges over the set $R$ of all $m!$ orders over the players (i.e. attributes) and $P_R^{a_j}$ is the set of players in $A$ which precede $a_j$ in the order $R$. Now, it follows from Equations (\ref{evidence-set-attributes}) and (\ref{cg-def-attributes}) that:
\begin{equation*}
\begin{split}
     \phi_{a}(a_j) &= \frac{1}{m!} \sum_{R} \sum_{X_i \in E_{a_j}} \Big[ \left|\mathscr{V}_a(P_R^{a_j}) \cup \{X_i\} \right| - \left|\mathscr{V}_a(P_R^{a_j})\right| \Big] \\
     &= \frac{1}{m!} \sum_R \sum_{X_i \in E_{a_j}} \mathcal{I}_{X_i \not\in P_R^{a_j}} = \frac{1}{m!} \sum_{X_i \in E_{a_j}} \sum_R \mathcal{I}_{X_i \not\in P_R^{a_j}} \\ 
     &= \sum_{X_i \in E_{a_j}} \frac{\sum_R \mathcal{I}_{X_i \not\in P_R^{a_j}}}{m!} = \sum_{X_i \in E_{a_j}} \frac{1}{|\{ k: X_i \in E_{a_k} \}|}
\end{split}
\end{equation*}
where $\mathcal{I}$ is the indicator function.
\end{proof}
The key take away from this lemma is that Shapley value of each attribute $a_j \in \{a_1,a_2, \ldots, a_m\}$ is an independent sum of {\em contributions} from its records wherein the contribution of each record is inversely proportional to the number of non-anomalous attributes it has. That is, {\em the higher is the Shapley value of an attribute, the more probable is that attribute being non-anomalous}.


\begin{example}
    We provide a demonstration of how to calculate Shapley values for attributes using a stylized tabular data with 5 attributes ($\{C_1,C_2,C_3,C_4,C_5\}$) and 6 records ($\{R1,R2,R3,R4,R5,R6\}$). 
    The evidence sets for the attributes are determined based on the placement of "$N\!A$" labels in the table. Specifically, $E_{C_1}=\{R1,R2,R4,R5\}$, $E_{C_2}=\{R2,R3,R6\}$, $E_{C_3}=\{R1,R3,R4,R5,R6\}$, $E_{C_4}=\{R1,R2,R5,R6\}$, and $E_{C_5}=\{R1,R3,R5\}$.

    To compute the Shapley value for attribute $C_4$ 
    we use the following formula:
    $ \phi_a(C_4) = \frac{1}{4} + \frac{1}{2} + \frac{1}{4} + \frac{1}{3} = 1.33,$ where $\frac{1}{4}$ corresponds to the appearance of $R1$ in four evidence sets ($E_{C_1}$, $E_{C_3}$, $E_{C_4}$, $E_{C_5}$), $\frac{1}{2}$ corresponds to the appearance of $R2$ in two evidence sets ($E_{C_1}$,  etc. Using a similar approach, we calculate the Shapley values for the other attributes: $\phi_a(C_1) = 1.33$, $\phi_a(C_2) = 1$, $\phi_a(C_3) = 1.66$, and $\phi_a(C_5) = 0.83$. Lower Shapley value scores are indicative of anomalous behavior. Based on the computed scores, we sort the attributes into four buckets in descending order of anomaly likelihood: ${C_5},{C_2},{C_1,C_4},{C_3}$.
\end{example}

The proposed framework is summarized in Algorithm \ref{algorithm:tab-shapley}, which we refer to as the {\em Tab-Shapley} algorithm. In this algorithm, Lines 4-5 calculate the evidence sets for each attribute using Definition 1. Then, Lines 7-11 efficiently compute the Shapley value for each attribute $a_j$ using the closed form expression described in Lemma \ref{labelled_cg_Shapley_value}.


\begin{algorithm}[t]
    \caption{Tab-Shapley Algorithm}
    \label{algorithm:tab-shapley}
    \KwIn{Tabular dataset $T$}
    \KwOut{Global scores for attribute in $T$ }
    \SetKwProg{Fn}{Function}{}{end}
    \Fn{\textsc{ComputeAttributeScores}($\{a_1, \cdots, a_m\}$)} 
    {
        $k\gets 0$ \;
        //construction of evidence sets \\
        \For{$j \in \{1, \ldots, m\}$}
        {
            $E_{a_j} \leftarrow \{X_i | L_{ij} = N\!A\ \  \mbox{ for any} \  i \in \{1,2,\ldots,n\} \} $ \;
        }
            
       //Computation of anomalous scores for attributes \\
        \For{$j \in \{1, \ldots, m\}$}
        {
            $s_{a_j} \leftarrow 0$ \\
            \For{$X_i$  in  $E_{a_j}$}
            {
                $s_{a_j} \leftarrow s_{a_j}$ + $\frac{1}{|\{ k: X_i \in E_{a_k} \}|}$
            }
        }
        \Return $\{s_{a_1}, s_{a_2}, \ldots, s_{a_m} \}$
    }
\end{algorithm}

\begin{figure*}[!t]
    \centering
    \subfloat[Arrhythmia:Original \label{fig:ar_or}]{{\includegraphics[width=0.20\textwidth]{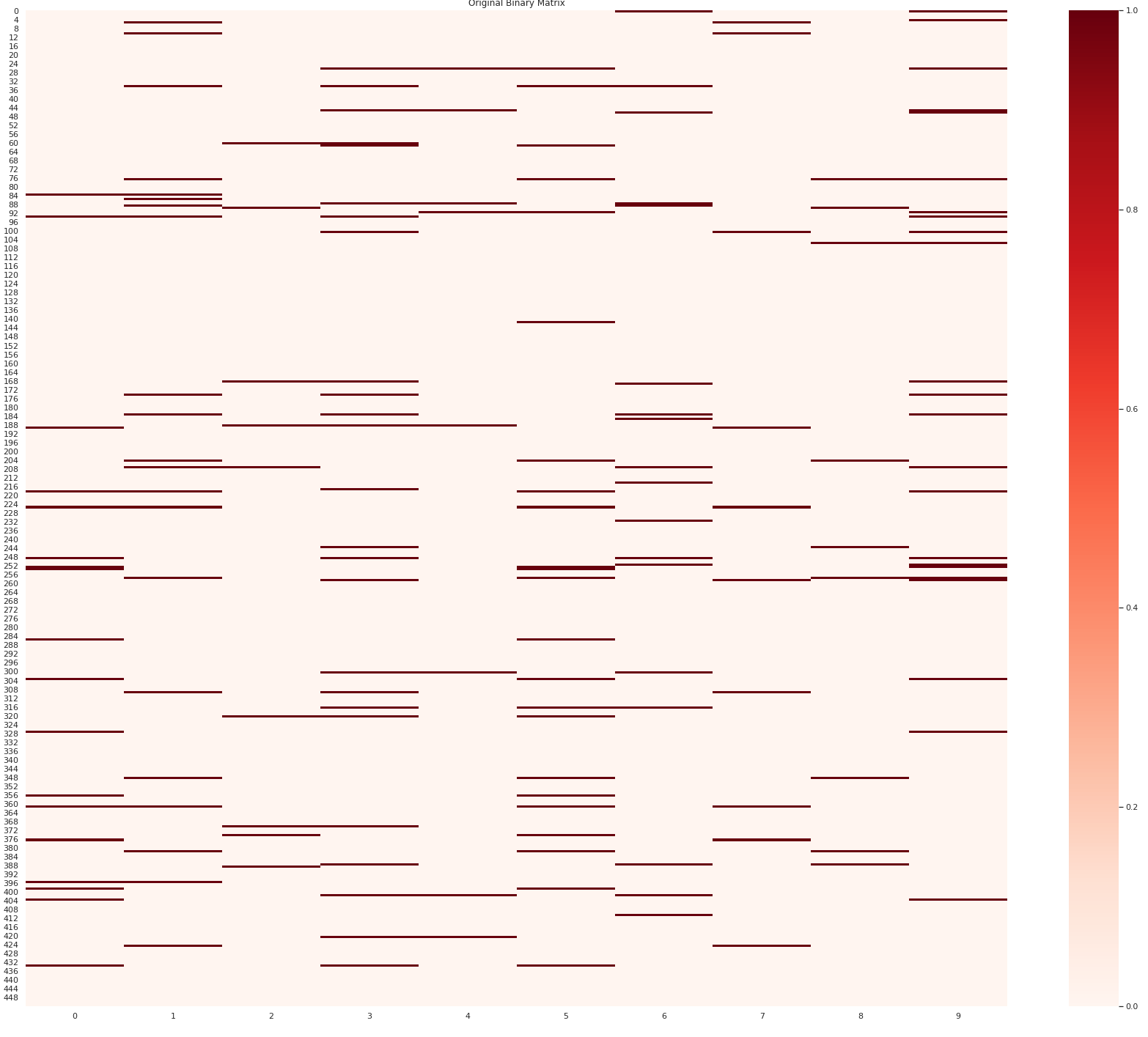} }}
    \subfloat[Arrhythmia:Tab-Shapley \label{fig:ar_ts}]{{\includegraphics[width=0.275\textwidth]{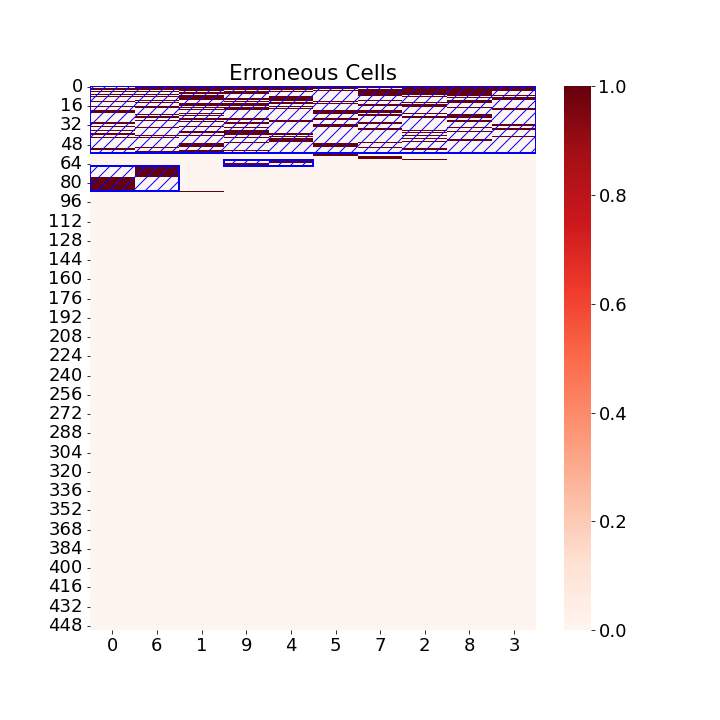} }}%
    \subfloat[Ionosphere:Original \label{fig:io_or}]{{\includegraphics[width=0.20\textwidth]{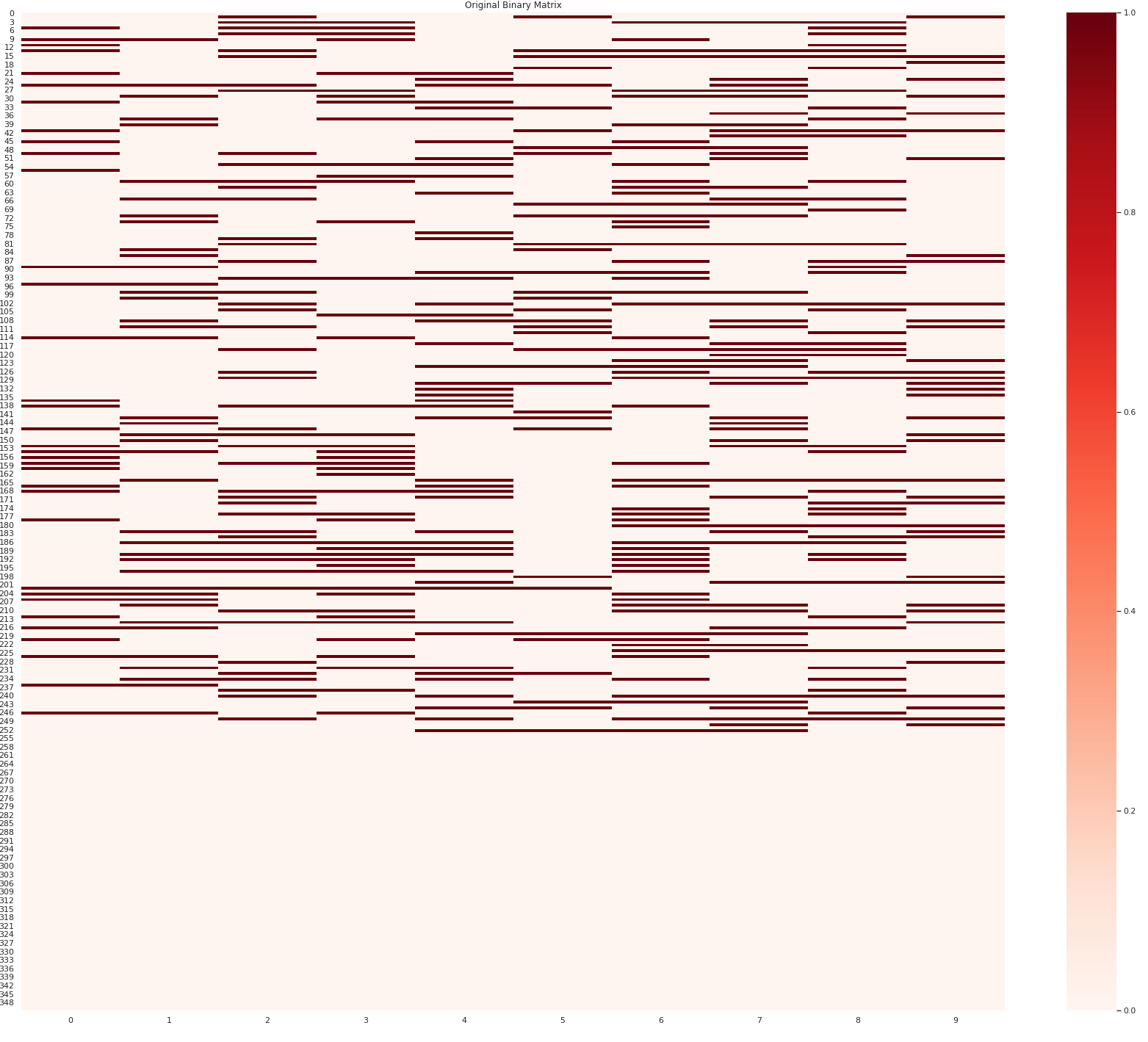} }}%
    \subfloat[Ionosphere:Tab-Shapley \label{fig:io_ts}]{{\includegraphics[width=0.275\textwidth]{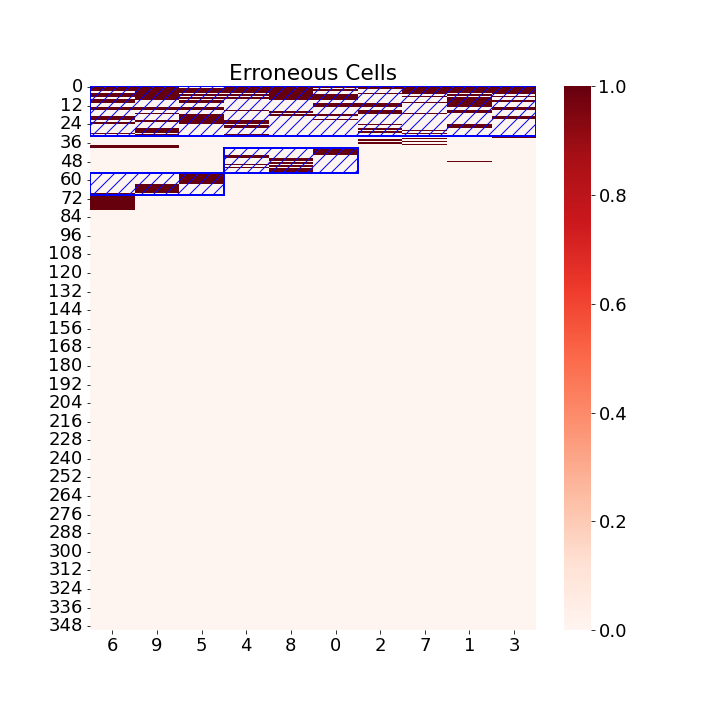} }}%
    \caption{Top-$k$ insights: Darker cells indicate anomaly. The blocks that are filled with blue patterns show the top-$K$ insights for $K=3$. The results are shown for $\alpha=0.2$; higher values of $\alpha$ would create smaller blocks.}%
\end{figure*}

\begin{figure}
    \centering
    \includegraphics[width=0.4\textwidth]{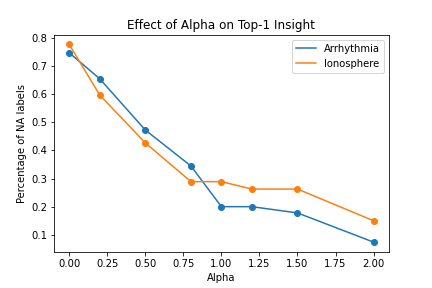}
    \caption{Impact of $\alpha$ on $N\!A$ cells contained in the top-$1$ insight.}
    \label{fig:alpha}
\end{figure}

\noindent{\bf Computational Complexity of Algorithm \ref{algorithm:tab-shapley}:} 
Now, let's analyze the time and space complexities of Algorithm \ref{algorithm:tab-shapley}. The construction of evidence sets $E_{a_j}$ for all $j$ requires $O(mn)$ time since we need to scan all the rows linearly for each attribute. In the worst case scenario, each evidence set can contain all the records, resulting in a space complexity of $O(mn)$.
Next, to compute the Shapley values, we utilize the closed-form expression derived in Lemma \ref{lemma2-proof-eqn1}. In Line 10 of Algorithm \ref{algorithm:tab-shapley}, we need to track the number of anomalous attributes in each record of the table, which requires $O(n)$ space. Once we have this information, the Shapley values can be computed linearly in $O(mn)$ time.
Overall, Algorithm \ref{algorithm:tab-shapley} can be executed with a time complexity of $O(mn)$ and a space complexity of $O(mn)$.

In some applications, the absolute values of reconstruction errors may provide useful information for identifying anomalous attributes.
To cater to them, we propose a weighted variant of the Tab-Shapley method which is directly computed on top of the reconstruction error.

\subsection{Computing Top-$k$ Data Insights}
First, we note that the attributes and records with lower Shapley values are more likely to be anomalous, allowing us to rearrange them accordingly and create block-like structures in the tabular data. This can be observed in Figures \ref{fig:ar_ts} and \ref{fig:io_ts}, where errors are concentrated in the top-left region of the table.

Now, we present Algorithm \ref{algorithm:top-k-DI-extraction} that outlines the details of extracting the top-$k$ data insights based on the block structures in the tabular data $T$. 
The algorithm depends on a scoring matrix $S$, that is of the same size as that of $T$. Each anomalous cell $S[i, j]$ is initialized with $1 - \frac{i \times j}{m \times n}$. The scores for anomalies thus have a decaying effect as we traverse toward the bottom-right of the table. The cells that are non-anomalous are initialized in a similar manner but with the sign flipped. We further scale the non-anomalous scores with a factor $\alpha > 0$, that controls the number of non-anomalous cells that we can afford to have in each insight. Figure \ref{fig:alpha} shows a declining trend for the percentage of non-anomalous cells with an increase of $\alpha$, which is as expected. Once the $S$ matrix is constructed, extracting top-$K$ insights simply reduces to iteratively extracting $K$ disjoint maximum sum subarrays from $S$, which can be efficiently solved using the Kadane's algorithm \cite{maxsum}. 

\begin{algorithm}[t]
    \caption{Extract top-$K$ data insights}
    \label{algorithm:top-k-DI-extraction}
    \KwIn{Reordered table $T$, integer $K$, $\alpha \in \mathbb{R}$, error labels $L$}
    \KwOut{ Top-$k$ data insights from $T$ }
    \SetKwProg{Fn}{Function}{}{end}
    \Fn{\textsc{ExtractInsights}($T, K, \alpha, L$)} 
    {
        let $n \leftarrow$ number of rows, $m \leftarrow$ number of columns in $T$ \;
        let $S_{n \times m} \leftarrow 0_{n\times m}$, insights $\leftarrow \{\}$ \;
        \For{$i \in  \{0, 1, \cdots, n-1\}$}
        {
            \For{$j \in \{0, 1, \cdots, m-1\}$}
            {
                \eIf{$L[i, j] == P\!A$}
                {
                     $S[i, j] \leftarrow 1 - \frac{i \times j}{n\times m}$ \\
                }
                {
                    $S[i, j] \leftarrow \alpha \times \left(\frac{i \times j}{n\times m} - 1 \right)$ \\
                }
                
            }
        }

    \For{$k \in K$}
    {
         insights[$k$] $\leftarrow$ \textsc{KadaneMaxSumSubArr}$(S)$ 
         
         Set $k^{\mbox{th}}$ submatrix in $S$ to $-\infty$ 
    }

    \Return $K$ sub-matrices of max sum in insights
    }
\end{algorithm}

\section{Experimental Results}
We evaluate the performance of the Tab-Shapley algorithm by comparing it with two baseline approaches: 1) DIFFI and 2) SHAP. Our experiments demonstrate that Tab-Shapley achieves more efficient ranking of attributes and rows compared to the baselines. As a result, the top-k insights derived using Tab-Shapley exhibit a higher concentration of errors, and thus help localize anomalies in the data. Additionally, we qualitatively analyze the Shapley values computed by the algorithm based on the two criteria discussed in Section \ref{evidence-sets-cgs}.\footnote{\footnotesize The anonymized code for the experiments is available at \url{https://drive.google.com/drive/folders/1CKxxBnBgHrY0fLwv7PZBD-ZluH-ld9nz?usp=sharing}}  

We first provide a description of the baseline approaches and then discuss the datasets used for comparison.




\smallskip
\noindent\textbf{Baselines.} We have used two popular approaches that are used to rank the features, (A) Global DIFFI algorithm \cite{carletti2019} uses isolation forest algorithm to derive a global ranking of features in an \textbf{unsupervised} manner (B) SHAP \cite{SHAP} is a \textbf{supervised} algorithm that determines the importance of each feature towards predicting the anomalous nature of records in the dataset.

Note that neither of the above baselines provides a ranking for rows in the dataset. In contrast, Tab-Shapley takes into account the anomalous behavior across {\em both rows and attributes}, allowing it to identify the most anomalous blocks of cells in the dataset. To ensure a fair comparison between the baselines and Tab-Shapley, we introduce a frequency-based row ordering for the baselines.

In the \emph{frequency-based} approach, we compute the total number of anomalous cells (i.e., "PA" cells) in each row and assign a lower rank to a row with a higher number of anomalous cells; in this way, we derive the top-$k$ insights using the baselines. 

\smallskip

\noindent\textbf{Datasets.} In our evaluation, we consider 12 real-world datasets ('Arrhythmia', 'Letter', 'Ionosphere' etc) that provide ground truth labels for both record-level and attribute-level anomalies \cite{xu2021}. These datasets were selected based on the work of Xu et al. \cite{xu2021}, where the authors explain the methodology used to obtain the ground truth labels. Specifically, we use the datasets generated using the probability-based method COPOD.

In addition to the labeled dataset, we also use two popular datasets that do not provide attribute-level ground truth information: (i) \textbf{KDD Cup 1999 Dataset} - We use the $10\%$ version of the data obtained from the UCI Machine Learning Archive, following a similar pre-processing approach as in \cite{antward21}. (ii) \textbf{Forest Cover Dataset.} - We apply a similar pre-processing approach as used by \cite{liu2008}.

\begin{table*}
    \centering
    \begin{tabular}{ccccccccccccccc}
    \toprule
    &\rotatebox[origin=c]{90}{KDD} & \rotatebox[origin=c]{90}{Forest} & \rotatebox[origin=c]{90}{OPT} & \rotatebox[origin=c]{90}{Speech} & \rotatebox[origin=c]{90}{Satimages} & \rotatebox[origin=c]{90}{WBC} & \rotatebox[origin=c]{90}{Arrhythmia} & \rotatebox[origin=c]{90}{Letter} & \rotatebox[origin=c]{90}{Ionosphere} & \rotatebox[origin=c]{90}{SPECT} & 
    \rotatebox[origin=c]{90}{Wine (w)} & \rotatebox[origin=c]{90}{Wine (r)} & \rotatebox[origin=c]{90}{Vertebral} & \rotatebox[origin=c]{90}{PIMA}  \\
    \midrule
    Criteria 1 &0.95 & 0.99& 0.92& 0.77& 0.93& 0.88& 0.78& 0.91& 0.96& 0.94& 0.99 & 0.80 & 0.97 & 0.95\\
    Criteria 2 &0.86& 0.99& 0.94& 0.93& 0.93& 0.96& 0.95& 0.94& 0.96& 0.99& 0.99& 0.87 & 0.89 & 0.94\\
    \bottomrule
    \end{tabular}
    \caption{Pearson's correlation coefficient between Shapley values of attributes and each of the two criteria using various datasets}
    \label{tab:shap-pearson-coefficient}
\end{table*}

\subsection{Efficiency of Top-$k$ Data Quality Insights}    
Here we evaluate the performance of the proposed Tab-Shapley framework to detect top-$k$ data quality insights using several real-world tabular datasets.

\noindent\textbf{Qualitative Evaluation.} We present the results for two datasets, Arrhythmia and Ionosphere, in Figures \ref{fig:ar_ts} and \ref{fig:io_ts} respectively. These figures illustrate the effectiveness of Tab-Shapley when applied using Algorithm \ref{algorithm:tab-shapley}, as it successfully concentrates the erroneous cells. Figures \ref{fig:ar_or} and \ref{fig:io_or} show the distribution of error cells before the Tab-Shapley-based aggregation, where the errors are spread throughout the matrix. However, after re-ranking and re-ordering the rows and columns based on the Shapley values, the erroneous rows and columns become concentrated in the top-left region, as depicted in Figures \ref{fig:ar_ts} and \ref{fig:io_ts}. These visualizations highlight the presence of distinct block structures, each representing a data quality insight. Such insights are invaluable for users in understanding the primary sources of anomalous attributes. We provide additional visualizations for the remaining datasets and synthetic datasets in the extended version.

\smallskip

\noindent\textbf{Quantitative Evaluation.} For quantitative evaluation, we compare Tab-Shapley with the baselines, DIFFI and SHAP. Both baselines provide attribute or column-level rankings, and we extend them by ranking the rows using the frequency-based approach to achieve overall aggregation. To assess the performance of Tab-Shapley-based aggregation in comparison, we compute the following metric.

First, we generate matrices that label each cell as PA (Possibly Anomalous) or NA (Not Anomalous) using the TabNet-based approach described in Section \ref{subsec:PANA}. We propose a metric that counts the number of PA cells in $k \times k$ blocks starting from the top-left. We assert that an effective aggregator will concentrate more anomalies in the top $k \times k$ blocks, thereby being more successful in providing the top-$k$ insights. Additionally, it is worth noting that the performance of algorithms tends to converge as the value of $k$ increases, as the coverage by the $k \times k$ block in the tabular data also increases with larger $k$ values.

\begin{figure}[!t]
    \centering
    \includegraphics[width=0.5\textwidth]{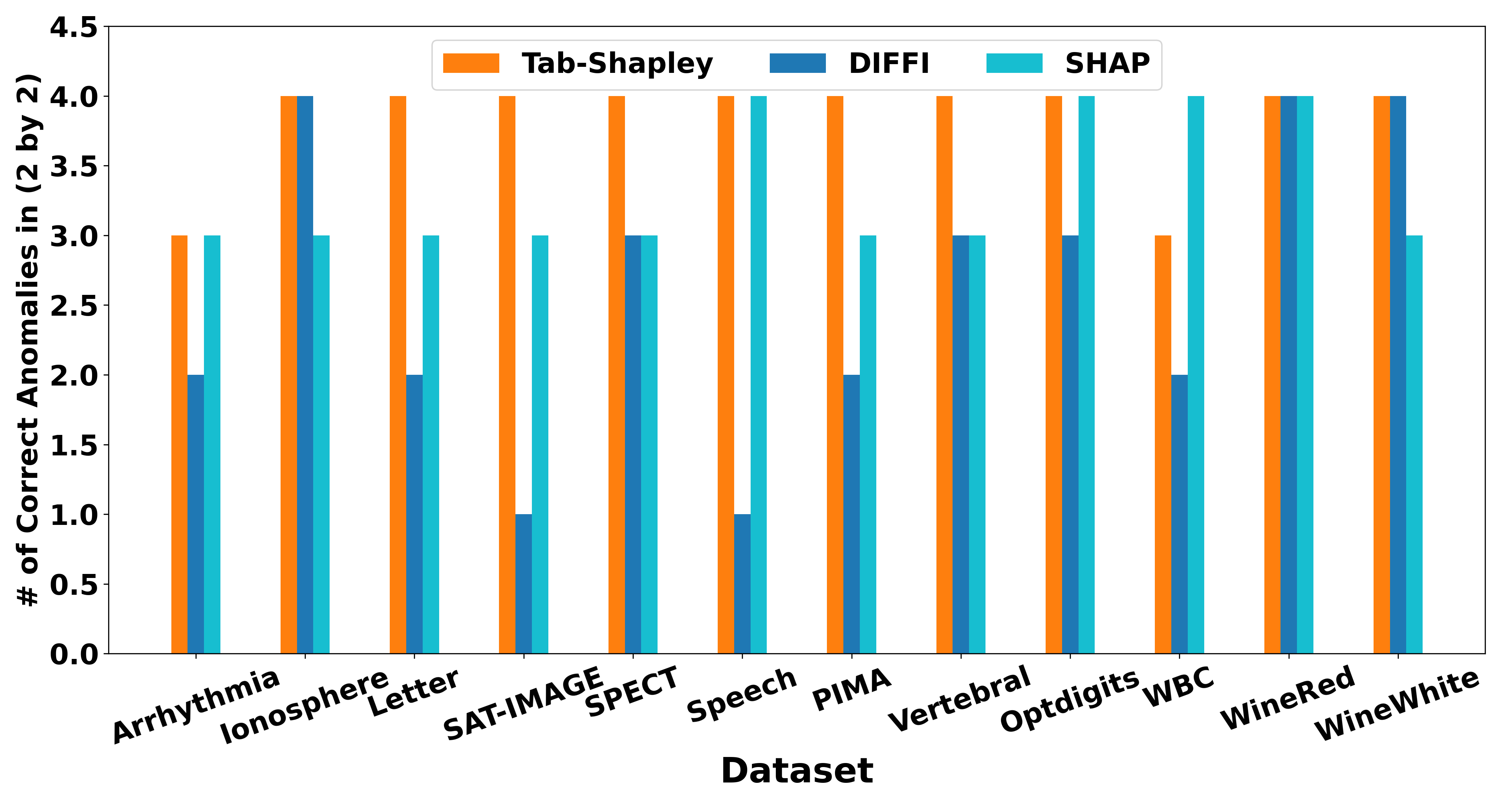}
    \caption{Number of ground-truth anomalies captured in the $2 \times 2$ sub-matrix using Tab-Shapley and DIFFI methods}     \label{fig:top-k-data-insights-1}
\end{figure}


Figure \ref{fig:top-k-data-insights-1} shows the number of ground-truth PA labels captured by the top-left $2 \times 2$ block for each dataset, comparing our proposed Tab-Shapley method with the baseline DIFFI. The bar chart clearly demonstrates the superior performance of the Tab-Shapley framework over the baselines DIFFI and SHAP. Similarly, Figure \ref{fig:top-k-data-insights-2} shows the number of PA labels captured by the respective top-left $4 \times 4$ block, further highlighting the superior performance of Tab-Shapley compared to DIFFI using real-world datasets. Figure \ref{fig:top-k-data-insights-3} extends this comparison to the top-left $6 \times 6$ block.

Note that, in many datasets, especially, in $4\times 4$ and $6 \times 6$ block we observe that SHAP outperforms Tab-Shapley. SHAP is supervised learning based approach, where the importance of each attribute is derived based on the target labels that classify every sample as anomalous or not. Hence, we show that, Tab-Shapley even in an unsupervised setting performs comparable to SHAP.

In summary, we conclude that Tab-Shapley offers valuable insights by effectively concentrating anomalies in them. This concentration of anomalies in turn offers an enhanced ability to understand the sources of anomalies within the dataset. Our proposed framework systematically evaluates the synergy effect on anomalous behavior when considering subsets of attributes.


\begin{table}
    \centering
    \begin{tabular}{cccccc}
    \toprule
    Dataset & Algorithm & 6*6 & 8*8 & 10*10 & 12*12 \\
    \midrule
    KDD Cup 1999&  Tab-Shapley & 17 & 24 & 38 & 46\\
           &  DIFFI & 10 & 16 & 22 & 31 \\
           &  SHAP & 26 & 36 & 54 & 70 \\
    \midrule
    Forest &  Tab-Shapley & 4 & 10 & 12 & 22\\
           &  DIFFI & 7 & 9 & 13 & 14\\
           &  SHAP & 0 & 2 & 3 & 12\\
    \bottomrule
    \end{tabular}
    \caption{Number of anomalies captured in the $6 \times 6$, $8 \times 8$, $10 \times 10$ and $12 \times 12$ submatrices by Tab-Shapley and DIFFI methods using KDD Cup 1999 and Forest Cover datasets.}
    \label{tab:KDD-Forest-Datasets}
\end{table}


We further evaluate the performance of the Tab-Shapley algorithm using two well-known datasets: KDD Cup 1999 and Forest Cover. Table \ref{tab:KDD-Forest-Datasets} presents the results, showcasing the number of anomalies captured within submatrices of size $6 \times 6$, $8 \times 8$, $10 \times 10$, and $12 \times 12$ by both the Tab-Shapley algorithm and the baselines for these datasets. We observe that, Tab-Shapley outperforms both the baselines in Forest Cover Dataset. In KDD Cup dataset, SHAP outperforms Tab-Shapley.

\begin{figure}[!t]
    \centering
    \includegraphics[width=0.45\textwidth]{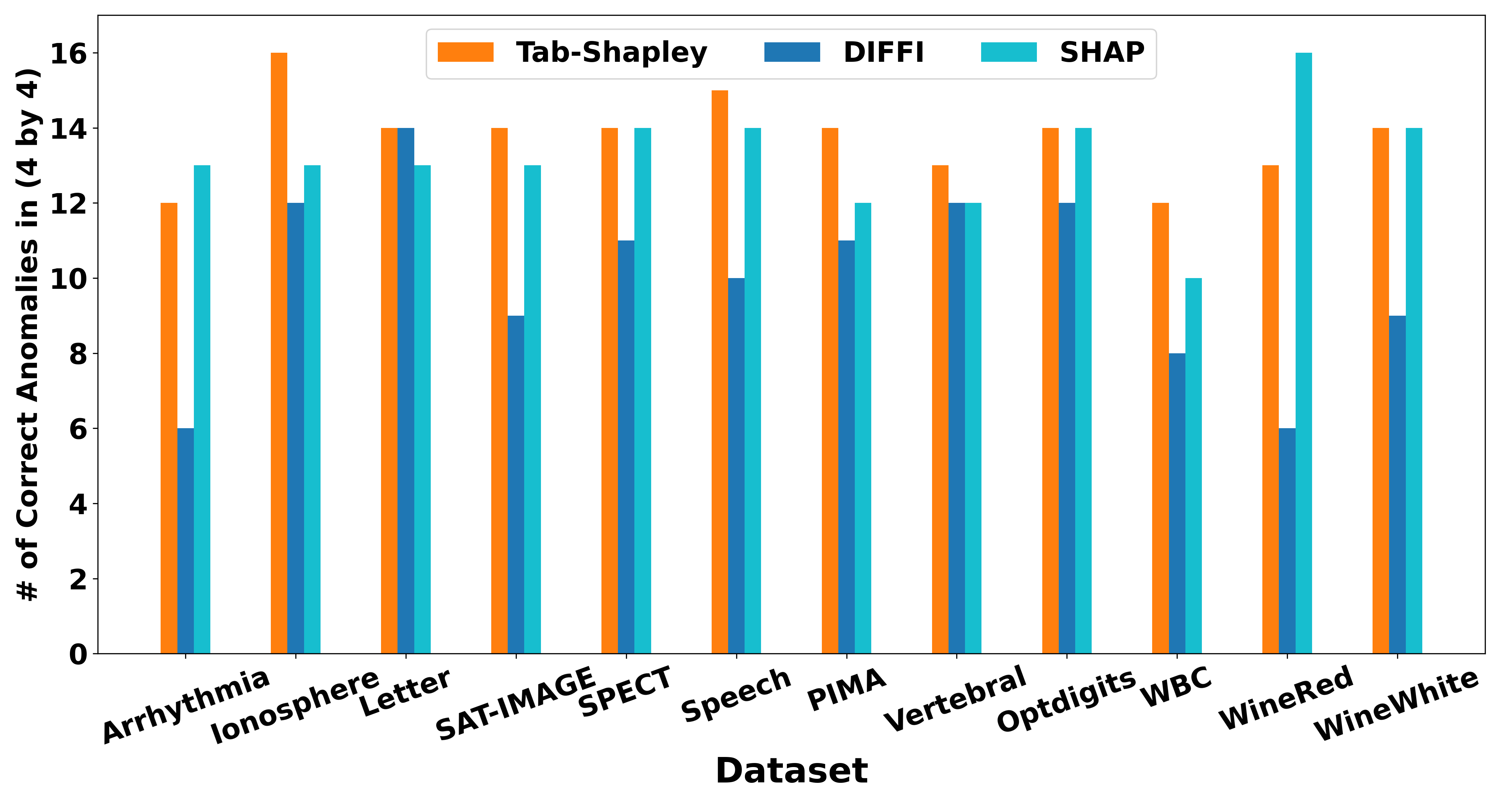}
    \caption{Number of ground-truth anomalies captured in the $4 \times 4$ sub-matrix using Tab-Shapley and DIFFI methods}
    \label{fig:top-k-data-insights-2}
\end{figure}

\begin{figure}[!t]
    \centering
    \includegraphics[width=0.45\textwidth]{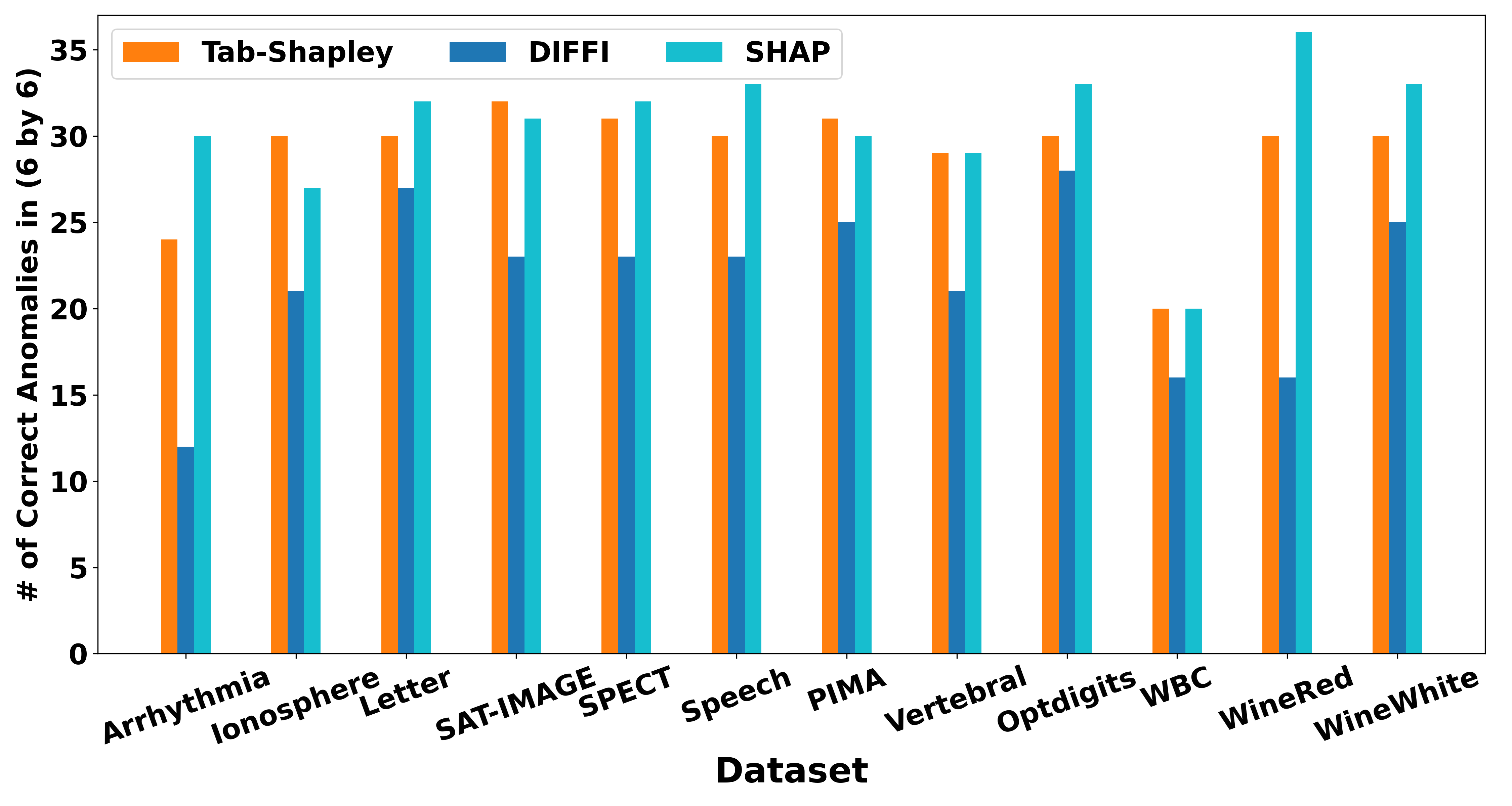}
    \caption{Number of ground-truth anomalies captured in the $6 \times 6$ sub-matrix using Tab-Shapley and DIFFI methods}
    \label{fig:top-k-data-insights-3}
    \vspace{-0.4cm}
\end{figure}


\subsection{Analysis of Shapley Values}


In this experiment, we delve deeper into the attribute scores derived using Shapley values to examine whether they meet the two criteria outlined in Section \ref{evidence-sets-cgs}. Recall that these criteria emphasize that (a) higher attribute scores should correspond to evidence sets with larger sizes and (b) a greater number of unique records. The first row of Table \ref{tab:shap-pearson-coefficient} presents the Pearson's correlation coefficient between the Shapley values of attributes and the sizes of their respective evidence sets (Criteria 1) across various datasets. Similarly, the second row of Table \ref{tab:shap-pearson-coefficient} presents the Pearson's correlation coefficient between the Shapley values of attributes and the number of unique records in their evidence sets (Criteria 2) across the datasets. The table shows a high Pearson's correlation, suggesting that our proposed Tab-Shapley approach aligns strongly with both the criteria.



\section{Conclusion and Future Work}

In this study, we introduced the novel problem of extracting "top-$k$ data quality insights" from tabular data and proposed the Tab-Shapley algorithm as an innovative solution. Our empirical analysis, conducted on both synthetic and real-world datasets, demonstrates the effectiveness of Tab-Shapley, surpassing the unsupervised baseline DIFFI and exhibiting comparable performance to the supervised baseline SHAP. Two potential avenues for future research include: (1) Conducting a human evaluation of Tab-Shapley on a large-scale industrial dataset to assess its real-world effectiveness, and (2) Exploring the integration of human feedback, gathered through annotations on anomalous blocks, to develop an online and adaptive version of the Tab-Shapley algorithm that refines generated insights based on human-in-the-loop feedback.


\bibliography{aaai25}

\end{document}